\documentclass[12pt,onecolumn,journal]{IEEEtran}

\usepackage{hyperref}       
\usepackage{url}            
\usepackage{booktabs}       
\usepackage{amsfonts}       
\usepackage{nicefrac}       
\usepackage{microtype}      
\usepackage{algorithm}
\usepackage{algorithmic}
\usepackage{enumerate}
\usepackage{enumitem}
\usepackage{bm}
\usepackage{subfig}

\usepackage{%
	anyfontsize,%
	etoolbox,%
	mathtools,%
	pgf,%
	pgfplots,%
	stmaryrd,%
	tikz,%
	xcolor,%
	tabto,%
    comment
}
\usepackage[capitalise]{cleveref}

\newtheorem{theorem}{Theorem}

\newtheorem{lemma}{Lemma}
\newtheorem{proof}{Proof}

\usepgfplotslibrary{patchplots}
\usepgflibrary{shapes}
\usetikzlibrary{%
	arrows,%
	decorations.text,%
	positioning,%
	scopes,%
	shapes%
}

\usepackage{setspace}
\usepgflibrary{shapes}
\usetikzlibrary{%
	arrows,%
	decorations.text,%
	positioning,%
	scopes,%
	shapes%
}
\pgfplotsset{compat=1.15}

\begin{document}
\title{Blind Decision Making: Reinforcement Learning with Delayed Observations}
\author{Mridul Agarwal, Vaneet Aggarwal \thanks{The authors are with Purdue University, West Lafayette IN 47907, USA, email:\{agarw180,vaneet\}@purdue.edu.}}

\maketitle

\begin{abstract}
Reinforcement learning typically assumes that the state update from the previous actions happens instantaneously, and thus can be used for making future decisions. However, this may not always be true. When the state update is not available, the decision taken is partly in the blind since it cannot rely on the current state information. This paper proposes an approach, where the delay in the knowledge of the state can be used, and the decisions are made based on the available information which may not include the current state information. One approach could be to include the actions after the last-known state as a part of the state information, however, that leads to an increased state-space making the problem complex and slower in convergence. The proposed algorithm gives an alternate approach where the state space is not enlarged, as compared to the case when there is no delay in the state update. Evaluations on the basic RL environments further illustrate the improved performance of the proposed algorithm.
\end{abstract}

\section{Introduction and Related Work}
A reinforcement learning (RL) agent models the world into states, which come from the set $\mathcal{S}$. The agent, at time $t$ and in state $s_t$, chooses an action $a_t \in \mathcal{A}$. After choosing an action, the agent receives a reward $r_t$ and moves to state $s_{t+1}$ at the next time step $t+1$. The goal of the agent is to find a policy to select an action that maximizes the total cumulative reward the agent collects over $T$ time steps, where $T$ may be unbounded \cite{sutton2018reinforcement}. Applications of RL are continuously increasing in domains that can be formulated using state, action, and rewards. Many such applications include cloud scheduling \cite{Arabnejad:2017:CRL:3101112.3101121}, robot manipulation \cite{clavera2018learning}, financial trading \cite{deng2017deep}, and microgrid management \cite{kuznetsova2013reinforcement}. However, many of these works assume that the state update is immediately observed after the action is played, which may not be true in general applications. This paper proposes an algorithm and evaluates it when the state updates are not immediately available. 

We note that in many applications, the state update information is not available instantaneously. As an example, micro-grid control may have stochastic delays because of the communication link and these delays may have adverse impact on the system \cite{liu2015impact}. The authors of \cite{changuel2012online} studied the impact of delays in video scheduling for mobile devices. For a mobile device, channel state information might be delayed since the streaming decisions are made many seconds before the content is played (to reduce the rebuffering events), which impacts the video streaming algorithms using RL based control. 360-degree video streaming further adds a challenge by introducing different qualities in different tiles in a frame, and thus the head movement prediction can be used for efficient streaming \cite{ghosh2017rate}. However, the head position while viewing a frame is not available for the streaming algorithm, which makes the decision before the content is played and is thus making decisions without the current state information. Similarly, cloud-based scheduling will observe stragglers only after a certain time threshold, and network delays may cause financial losses in trading. Delays can also arise from processing delays on input images or computationally complex data processing.  The issue of such delay in the availability of state information limits the use of RL for practical applications \cite{mahmood2018setting}. 

As studied by \cite{katsikopoulos2003markov}, delays may be of three types, 1) observation delays, where observations for state updates are delayed, 2) action delays, where actions may take effect after a certain delay, and 3) cost/reward delays, where rewards are not obtained till certain time after an action is taken in any state. They show that action delays are equivalent to observation delays. If observations are delayed then the agent plays action with knowledge of the last observed state and the actions happened in the meantime. Similarly, if actions are delayed, the agent schedules actions for the future with the same information. In case of delays in rewards updates, algorithms train using mini-batches where a mini-batch consists of state observations, actions taken, and available rewards received for some duration. We note that even though the current state information is not known, the decisions still have to be made. In this paper, we assume only observation or action delays. We further assume that rewards are provided to the agent along with state updates.

The authors of \cite{altman1992closed} showed that for Markov Decision Process (MDP) where each observation is delayed by $d$ steps, an equivalent MDP can be constructed with an augmented state space where previous $d$ actions are appended to the currently known state. Thus, the new state space $\mathcal{S}' = \mathcal{S}\times \mathcal{S}^d$ with $s_t'= (s_{t-d}, a_{t-d}, a_{t-(d-1)}, \cdots, a_{t-1})$. This formulation restores the problem structure back to an MDP, and any RL algorithm can now be applied to the new MDP. We point to three limitations of this approach.  

The first limitation is that this approach does not scale to stochastic delays. In order to resolve this, \cite{katsikopoulos2003markov} proposed a new solution by assuming that the maximum delay is bounded. They assume that if the delays are more than a threshold $n$, then the algorithm freezes, and would not take any action. However, many real-time systems might not enjoy this liberty of freezing the execution of the algorithm.

The second limitation is that the expansion of the state space might not be efficient for implementation because of the increased storage complexity and exponentially larger time to converge to the optimal solution. In order to alleviate this, the authors of \cite{walsh2009learning} proposed an algorithm to play action $a_t$ which is optimal for the state in which the agent is most likely in, or $a_t = \arg\max_{a_t} Q(\arg\max{p_{s_t}(s_t}|s'_t), a_t)$. It was assumed that the probability of not being in the most likely state is bounded by $\epsilon$, where $\epsilon$ is small enough. Note that for Markov Decision Process (MDP) where state distributions are not concentrated, this assumption might not hold.

The third limitation for augmenting the state space is that the regret analysis of the MDP shows that the regret of this algorithm scales as  $|\mathcal{A}|^{d/2}$, where $d$ is the delay in the availability of the state information \cite{jin2018q}. For large $d$, the gap may be large enough for the approach to have significantly decreased performance.


Having pointed out the limitations in the prior works, we propose a solution that aims to alleviate these issues. The proposed algorithm, called {\it Expectation-Maximization Q-learning} (EMQL), takes an action that maximizes the expected gain of true MDP across all possible states conditioned over the last known state $s_{t-d}$, and actions taken till time $t$. EMQL for delayed reinforcement learning has the following properties:
\begin{itemize}
    \item {\bf Space efficient}: The proposed algorithm does not use an augmented MDP to determine the action. It, however, uses the fact the current true state comes from the probability distribution generated by augmented MDP.
    \item {\bf Robust under deviation from most likely state}: The algorithm selects the action that maximizes the expected value $Q$ function. Thus, even though the distribution is not concentrated around a single state, the distribution is efficiently utilized.
    \item {\bf Handles stochastic delays:} The algorithm works well with stochastic delays, as well as missed information.
\end{itemize}

EMQL is evaluated on Frozen Lake ($8\times 8 grid)$, and Cart Pole environments of OpenAI Gym platform \cite{Brockman2016OpenAIG}. The results for delayed settings are compared with Extended MDP formulation of \cite{altman1992closed}, MBS algorithm given by \cite{walsh2009learning}, and dQ algorithm proposed by \cite{schuitema2010control} respectively. The metric of comparison is the total reward, collected in each episode.

The rest of the paper is organized as follows. 
Section \ref{forumulation} defines the problem under consideration. Section \ref{algorithm} presents the proposed algorithm. Evaluation and comparison results are presented in section \ref{evaluation}. Section \ref{sec:concl} concludes the paper with a brief discussion.  
\section{Formulation} \label{forumulation}
We consider a Markov Decision Process $\mathcal{M}$, with set of states denoted by $\mathcal{S}$, and set of actions denoted by $\mathcal{A}$. At time $t$, the environment is in state $s_t \in \mathcal{S}$. The definitions are mostly consistent with those in \cite{sutton2018reinforcement,puterman2014markov}

At any time $t$, the agent chooses action $a_t \in \mathcal{A}$ based on its knowledge about the current state. On playing the action $a_t$, environment rewards the agent with $R_t$, which is random variable conditioned on environment state $s_t$, and action chosen by agent at time $t$. The maximum reward the agent can receive at any time step is $R_{max}$. The goal of the agent is to maximize the discounted cumulative rewards it receives. The discount factor $\gamma\in[0,1)$ denotes the importance of future rewards.
\begin{align}
    \mathcal{R} &= \sum_{t=0}^\infty \gamma^t R_t
\end{align}

Expected reward when action $a$ is taken in a state $s$ is defined as $r(s,a)$
\begin{align}
    r(s,a) = \mathbb{E}\left[R_t|s_t = s, a_t = a\right]
\end{align}

The probability distribution of next state $s_{t+1}$ conditioned on current state $s_t$ and action $a_t$ is denoted by $p(s_t, a_t, s_{t+1})$. Shorthand notation by dropping the subscripts is denoted as,
\begin{align}
    p(s, a, s') &= P(s_{t+1} = s'|s_t = s, a_t =a)
\end{align}
Agent uses a policy $\pi$ to select an action. $\pi$ is defined as the probability distribution over actions given the state.
\begin{align}
    \pi(a|s) = \mathbb{P}(a_t = a|s_t = s)
\end{align}
The value function $V^\pi(s)$ of a state $s$ is defined as the expected value of sum of discounted rewards which agent can receive over time starting from state $s$ and choosing actions according to the policy $\pi$.
\begin{align}
    V^\pi(s) &= \mathbb{E}_{\pi}\left[\sum_{t=t_0}^\infty\gamma^{t-t_0}R_t\Big|s_{t_0} = s\right]
\end{align}

This makes the maximum possible value of $V^\pi(s)$ as $\frac{R_{max}}{1-\gamma}$. Similarly action-value function $Q^\pi(s,a)$ is defined as the expected cumulative rewards which agent receives in state $s$ on taking action $a$ and then following policy $\pi$,
\begin{align}
    Q^\pi(s) &= \mathbb{E}_{\pi}\left[\sum_{t=t_0}^\infty\gamma^{t-t_0}R_t\Big|s_{t_0} = s, a_{t_0} = a\right]
\end{align}

For both value function $V^\pi(s)$ and action-value function $Q^\pi(s,a)$ the expectation is taken over the states which are distributed according to the transition dynamics of the MDP $\mathcal{M}$ and actions which are distributed according to the policy $\pi$. We use only $\pi$ in the subscript for expectation as we can only control the policy. Optimal policy $\pi^*$ is defined as the policy which maximizes the value function for all states.
\begin{align}
    V^{*}(s) &= V^{\pi^*}(s)\\
    &= \sup_{\pi} V^\pi (s)\ \ \ \forall\ s\in\mathcal{S}
\end{align}

The optimal policy $\pi^*$ gives an optimal value function $V^*$ and an optimal Q-function $Q^*$, which are related as
\begin{align}
    Q^*(s,a) &= \mathbb{E}_{\pi}\left[r(s_t,a_t)\right] + \gamma\mathbb{E}\left[V^*(s_{t+1})|s_t = s, a_t =a \right]\\
    V^*(s) &= \max_{a\in\mathcal{A}}Q^*(s,a)\ \ \ \forall\ s\in\mathcal{S}
\end{align}

In practical RL algorithms, agent deploys some strategy to calculate an estimate of Q-function $\hat{Q}$. Some common strategies are discussed in \cite{sutton2018reinforcement}. The agent selects the action greedily based on the estimate of Q-function $\hat{Q}$ as
\begin{align}
    a_t = \arg\max_{a_t \in \mathcal{A}} \hat{Q}(s_t, a_t)
\end{align}

The delay $\Delta$ is a random variable denoting the delay of the system. All realizations of delay $d$ are assumed to be a non-negative integer. At time $t$, the last known state for the agent is $s_{t-d}$. The actions played in $d$ time steps are $a_{t-d}, a_{t-(d-1)}, \cdots, a_{t-1}$.

We assume that at the beginning of any episode, all the delayed observations of previous episodes are available. This also means that the observations from any of the previous episodes are not corrupting the observations received in current episode.
\section{Proposed Policy and Bounds} \label{policy}
For a system with $d$ delays, we construct an extended MDP $\widetilde{\mathcal{M}}$ which has state $\tilde{s}_t$ as $(s_{t-d}, a_{t-d}, a_{t-(d-1)}, \cdots, a_{t-1})$. The two MDPs $\widetilde{\mathcal{M}}$, and $\mathcal{M}$ share the same action space $\mathcal{A}$, so we will not change the notation for actions. For everything else we will put a tilde over the variables for the augmented MDP. The policy $\tilde{\pi}$ now selects an action based on $\tilde{s}$ or the tuple $(s_{t-d}, a_{t-d}, a_{t-(d-1)},\cdots,a_{t-1})$. The corresponding Q-function for a policy $\tilde{\pi}$ over $\widetilde{\mathcal{M}}$ becomes,
\begin{align}
    \widetilde{Q}^{\tilde{\pi}}(\tilde{s}, a) &= \mathbb{E}_{\tilde{\pi}}\left[R_t|\tilde{s}_t=\tilde{s},a_t=a\right] +\nonumber \\
    &\gamma\sum_{\tilde{s}_{t+1}\in\mathcal{\widetilde{S}}}\widetilde{V}^{\tilde{\pi}}(\tilde{s}_{t+1})\mathbb{P}\left[\tilde{s}_{t+1}|\tilde{s}_t = \tilde{s}, a_t = a\right]
\end{align}

Using this construction we present the key lemma based on which we construct our policy.
\begin{lemma}\label{key_lemma}
Expected reward obtained by agent in augmented state $\tilde{s}_t$ by taking an action $a$, is related to the true state $s_t$ of environment as
\begin{small}
\begin{align}
    \tilde{r}(\tilde{s}_t, a) &= \sum_{s\in\mathcal{S}}r(s, a)p(s|\tilde{s}_t)\label{eq:reward_equivalence}
\end{align}
\end{small}
\begin{proof}
Reward $R_t$ generated by the environment is oblivious to the state maintained by the agent. The expected reward for the agent is $\tilde{r}(\tilde{s}_t,a)$
\begin{small}
\begin{align}
    \tilde{r}(\tilde{s}, a) &= \mathbb{E}\left[R_t|\tilde{s}_t=\tilde{s}, a_t = a\right]\\
    &= \sum_{s\in\mathcal{S}}\mathbb{E}\left[R_t|s_t = s,\tilde{s}_t=\tilde{s}, a_t = a\right]\mathbb{P}(s_t = s|\tilde{s}_t = \tilde{s})\nonumber\\
    &=\sum_{s\in\mathcal{S}}\mathbb{E}\left[R_t|s_t = s, a_t = a\right]\mathbb{P}(s_t = s|\tilde{s}_t = \tilde{s}) \label{eq:reward_only_env_state}
\end{align}
\end{small}
Equality (\ref{eq:reward_only_env_state}) follows from the fact that reward does not depend on the state maintained by agent, but on the environment's state.
\end{proof}
\end{lemma}

Lemma \ref{key_lemma} states that the expected reward received on taking action $a$ in state $\tilde{s}$ is the expected reward received by taking action $a$ in the unobserved stated conditioned on $\tilde{s}$. Based on Lemma \ref{key_lemma}, a myopic policy, which maximizes immediate expected return for the agent, selects greedy action $a_t(\tilde{s})$ as
\begin{small}
\begin{align}
    a_t = \arg\max_{a\in\mathcal{A}}\left(\mathbb{E}_s\left[R(s,a)|\tilde{s}_t\right]\right)
\end{align}
\end{small}
Inspired by the myopic policy, we now propose a policy for working with delayed state updates.

The agent would be able to maximize its expected discounted cumulative rewards if it has oracle access which could return the optimal action for unobserved current state $s_t$ of the environment. However, since such oracle access is not available, we settle for a policy that assumes that the state at the next time step will be available. This policy maximizes the sum of immediate reward and the expected value of the next state under the optimal policy for $\mathcal{M}$, or
\begin{align}
    \tilde{\pi}(a_t|\tilde{s}_t)= \begin{cases}  
                1, & \text{if } a_t = \arg\max_{a\in\mathcal{A}}\mathbb{E}\left[Q^*(s,a)|\tilde{s}_t\right],\\ 
                0, & \text{otherwise }
            \end{cases}\label{eq:final_policy}
\end{align}

The following theorem provides bounds on minimum value an augmented state $\tilde{s}\in\widetilde{\mathcal{S}}$ would fetch for the agent. That is, value function using the policy defined in Equation (\ref{eq:final_policy}) ensures the minimum value given in Theorem \ref{final_theorem}.
\begin{theorem}\label{final_theorem}
If the agent follows policy as given in (\ref{eq:final_policy}), for an augmented MDP $\widetilde{\mathcal{M}}$, then the value of each state $\tilde{s}\in\widetilde{\mathcal{S}}$ satisfies the following lower bound,
\begin{align}
    \widetilde{V}^{\tilde{\pi}}(\tilde{s}) &\geq \mathbb{E}_{s|\tilde{s}}\left[{V}^{\pi^*}(s)\right] -\frac{R_{max}}{(1-\gamma)^2}\left(1-\frac{1}{|\mathcal{A}|}\right)
\end{align}

where $\pi^*(a_t|s_t)$ is oracle aided policy which gives the optimal action for true MDP $\mathcal{M}$.
\end{theorem}
\begin{proof}
We first mention and prove the lemmas required for the proof, and then continue to the final proof.
Lemma \ref{lemma_2} relates the value function for $\mathcal{M}$ and $\widetilde{\mathcal{M}}$ under same policy $\tilde{\pi}$. Lemma \ref{lemma_3} relates expected value using the oracle aided policy and the state-action value using oracle aided policy.
\begin{lemma}\label{lemma_2}
Value function of a policy $\tilde{\pi}(a|\tilde{s})$ for augmented MDP $\tilde{\mathcal{M}}$ is related to value function for same policy under true MDP $\mathcal{M}$ is related as
\begin{align}
    \tilde{V}^{\tilde{\pi}}(\tilde{s}) &= \mathbb{E}_{s}\left[V^{\tilde{\pi}}(s)|\tilde{s}\right]
\end{align}

where $V^{\tilde{\pi}}(s)$ is value function for policy $\tilde{\pi}$ with MDP $\mathcal{M}$.
\end{lemma}
\begin{proof}
Note that the value function for the augmented MDP $\tilde{\mathcal{M}}$ is expected cumulative discounted rewards collected by the agent when it starts from state $\tilde{s_t}$, and follows a policy $\tilde{\pi}$. Thus from the definition of $\tilde{V}^{\tilde{\pi}}(\tilde{s})$, we have
\begin{align}
    \tilde{V}^{\tilde{\pi}}(\tilde{s}) &= \mathbb{E}_{\tilde{\pi}}\left[\sum_{k=0}^{\infty}\gamma^{t+k}R_{t+k}|\tilde{s}_t=\tilde{s}\right]\\
    &= \sum_{k=0}^{\infty}\gamma^k\mathbb{E}_{\tilde{\pi}}\left[R_{t+k}|\tilde{s}_t=\tilde{s}\right]\label{eq:lin_exp_val_fn}\\
    &= \sum_{k=0}^{\infty}\gamma^k\mathbb{E}_{\tilde{\pi}}\left[R_{t+k}|s_{t+k}, a_{t+k}\right]\mathbb{P}\left[s_{t+k}|\tilde{s}_t, \tilde{\pi}\right]\label{eq:use_lemma_1}\\
    &= \sum_{k=0}^{\infty}\gamma^k\mathbb{E}_{\tilde{\pi}}\left[R_{t+k}|s_{t+k}, a_{t+k}\right]\sum_{s_t\in\mathcal{S}}\mathbb{P}\left[s_{t+k}|s_t, \tilde{\pi}\right]\mathbb{P}\left[s_t|\tilde{s}_t \right]\label{eq:use_iter_exp}\\
    &= \sum_{s_t\in\mathcal{S}}\left(\sum_{k=0}^{\infty}\gamma^k\mathbb{E}_{\tilde{\pi}}\left[R_{t+k}|s_{t+k}, a_{t+k}\right]\mathbb{P}\left[s_{t+k}|s_t, \tilde{\pi}\right]\right)\mathbb{P}\left[s_t|\tilde{s}_t \right]\nonumber\\
    &= \sum_{s_t\in\mathcal{S}}V^{\tilde{\pi}}(s)\mathbb{P}\left[s_t|\tilde{s}_t \right]
\end{align}

Equation \eqref{eq:use_lemma_1} follows from Lemma \ref{key_lemma}. Equation \eqref{eq:use_iter_exp} follows the fact that if the agent knew the true state and followed the policy $\tilde{\pi}$, then the environment state evolution probabilities would remain the same.
\end{proof}
\begin{lemma}\label{lemma_3}
For the optimal policy $\pi^*$, and distribution over initial state $s$, $\mu$,
\begin{align}
    \arg\max_{a}\mathbb{E}_{s\sim\mu}\left[Q^*(s,a)\right] \geq \frac{1}{\mathcal{A}}\mathbb{E}_{s\sim\mu}\left[V^*(s)\right]
\end{align}

\end{lemma}
\begin{proof}Let initial state follow some distribution $\mu$, or $\mathbb{P}\left[s_t=s\right]\sim \mu$. Then we have,
\begin{align}
    \mathbb{E}_{s\sim\mu}\left[V^*(s)\right] &= \mathbb{E}_{s\sim\mu}\left[\max_a Q^*(s,a)\right]\\
                                             &\leq \mathbb{E}_{s\sim\mu}\left[\sum_{a\in\mathcal{A}} Q^*(s,a)\right]\\
                                             &= \sum_{a\in\mathcal{A}} \mathbb{E}_{s\sim\mu}\left[Q^*(s,a)\right]\\
                                             &\leq |\mathcal{A}|\max_{a\in\mathcal{A}}\mathbb{E}_{s\sim\mu}\left[Q^*(s,a)\right]
\end{align}
\end{proof}

We can now use Lemma \ref{lemma_2}, Lemma 6.1 from \cite{kakade2002approximately} and Lemma \ref{lemma_3} to find the minimum value of $\widetilde{V}^{\tilde{\pi}}(\tilde{s}) - \mathbb{E}_s\left[V^{\pi^*}(s)|\tilde{s}\right]$. Mathematically we have, 
\begin{align}
    &\tilde{V}^{\tilde{\pi}}(\tilde{s}) - \mathbb{E}_s\left[V^{\pi^*}(s)|\tilde{s}\right] =  \mathbb{E}_s\left[V^{\tilde{\pi}}(s)|\tilde{s}\right] - \mathbb{E}_s\left[V^{\pi^*}(s)|\tilde{s}\right]\label{eq:exp_over_values}\\
                                         &= \mathbb{E}_{s|\tilde{s}, s_{t+1},a_t\sim\tilde{\pi}}\left[\sum_{k=0}^{\infty}\gamma^k{Q}^{\pi^*}(s_{t+k},a_{t+k})-{V}^{\pi^*}(s_{t+k})\right]\label{eq:using_kakade_lemma}\\
                                         &= \sum_{k=0}^{\infty}\gamma^k\mathbb{E}_{s|\tilde{s}, s_{t+1},a_t\sim\tilde{\pi}}\left[{Q}^{\pi^*}(s_{t+k},a_{t+k})\right]\nonumber\\
                                         &\ \ \ \ -\mathbb{E}_{s|\tilde{s}, s_{t+1},a_t\sim\tilde{\pi}}\left[{V}^{\pi^*}(s_{t+k})\right]\\
                                         &\geq \sum_{k=0}^{\infty}\gamma^k\frac{1}{|\mathcal{A}|}\mathbb{E}_{s|\tilde{s}, s_{t+1},a_t\sim\tilde{\pi}}\left[{V}^{\pi^*}(s_{t+k})\right]\nonumber\\
                                         &\ \ \ \ -\mathbb{E}_{s|\tilde{s}, s_{t+1},a_t\sim\tilde{\pi}}\left[{V}^{\pi^*}(s_{t+k})\right]\label{eq:using_lemma_3}\\
                                         &\geq \sum_{k=0}^{\infty}\gamma^k\left(\frac{1}{|\mathcal{A}|}-1\right)\mathbb{E}_{s|\tilde{s}, s_{t+1},a_t\sim\tilde{\pi}}\left[{V}^{\pi^*}(s_{t+k})\right]\\
                                         &\geq \sum_{k=0}^{\infty}\gamma^k\left(\frac{1}{|\mathcal{A}|}-1\right)\frac{R_{max}}{1-\gamma} = \left(\frac{1}{|\mathcal{A}|}-1\right)\frac{R_{max}}{(1-\gamma)^2} \label{eq:bound_min}
\end{align}

Equation (\ref{eq:exp_over_values}) follows from Lemma \ref{lemma_2}, and Equation (\ref{eq:using_kakade_lemma}) comes from using Lemma 6.1 from \cite{kakade2002approximately}. Equation (\ref{eq:using_lemma_3}) comes from the fact that $a_{t+k}$ is chosen from the policy defined in Equation (\ref{eq:final_policy}) for all $k\geq 0$, and Lemma \ref{lemma_3}. Equation (\ref{eq:bound_min}) follows from maximum possible value of $V^\pi(s)$.
\end{proof}

Theorem \ref{final_theorem} states that the proposed policy can suffer a maximum degradation of $\left(\frac{1}{|\mathcal{A}|}-1\right)\frac{R_{max}}{(1-\gamma)^2}$ only from the expected optimal value of unobserved state conditioned on the extended state.

Now, the task that remains is to find $Q^*$ for true MDP. We assume that the delayed state observations can be identified using timestamp or index header. This is a common engineering principle in communication networks to deal with asynchronous packets \cite{walrand2010communication}, and hence it is a valid assumption. This allows to find optimal Q-function for true MDP $\mathcal{M}$. We next provide a detailed algorithm  for the policy described in this section.
\vspace{-.1in}
\section{Algorithm} \label{algorithm}
We now utilize Equation (\ref{eq:final_policy}) to construct {\it Expectation Maximization Q Learning} (EMQL) algorithm (described in Algorithm \ref{alg:EQL}) which is space efficient and which can handle stochastic delays. We note that to calculate the expected Q-value function, we require an algorithm to calculate and store the Q values for each state-action pair. Also the algorithm requires to estimate and store the state transition probabilities to calculate the expected Q value for Equation (\ref{eq:final_policy}). For this, we divide our algorithm into two parts. First part (Algorithm \ref{alg:EQL}) keeps a track of visited state-action pairs, and observed next states and rewards to calculate the Q-value table and the probability transition matrix. Second part is an auxiliary algorithm (Algorithm \ref{alg:Agent_get_action}) which actually implements the policy of Equation (\ref{eq:final_policy}) using the Q-value table and probability transition matrix of the first part. 
\vspace{-.1in}
\subsection{Algorithm Construction} \label{alg_const}
The algorithm takes the state space $\mathcal{S}$, action space $\mathcal{A}$, discount factor $\gamma$, and exploration factor $\epsilon_t$ as inputs. Since, we consider a model based algorithm, we maintain variables corresponding to number of times a state-action pair was visited $(N(\cdot, \cdot))$, rewards obtained for the state-action pair $(R(\cdot, \cdot))$, and the counter for next state from a state-action pair to calculate the estimates of probability transitions $(P(\cdot,\cdot,\cdot))$.

Before beginning any episode the algorithm calculates the estimated probability transitions $(\hat{p}(\cdot, \cdot, \cdot))$ and the expected rewards $(\hat{r}(\cdot, \cdot))$ according the equations (\ref{eq:udpate_prob}) and (\ref{eq:udpate_reward}) respectively. It then updates the Q-function $(Q(\cdot, \cdot))$ for the true MDP as per the equation (\ref{eq:update_Q}). The Q-function learned is for $\mathcal{M}$, which converges faster because of the smaller state space.

In each episode, the algorithm follows an $\epsilon$-greedy approach for exploration. It generates a random number $\mathbb{X}$ from uniform distribution over $[0,1]$. If the random number generated is less than $\epsilon_t$, it plays action randomly uniformly from the action space $\mathcal{A}$. Else the algorithm uses the auxiliary algorithm {\it Get\_EMQL\_Action} to determine the action to be played using the proposed policy of Equation \eqref{eq:final_policy} in Section \ref{policy}. After playing an action, if a new observation is available to the algorithm, it updates the following parameters - the number of times state $s$ was visited and action $a$ was taken in $s$ $N(s, a)$, number of times $s'$ was observed as the next state $P(s, a, s')$,  and reward observed for the state action pair $R(s, a)$.

{\it Get\_EMQL\_Action} is presented in Algorithm \ref{alg:Agent_get_action}, and is described in subsection \ref{complexity}. If the current state is available, the expected Q-function becomes the true Q function. Thus Algorithm \ref{alg:EQL} evaluates Q-value table and probability estimates, and Algorithm \ref{alg:Agent_get_action} computes the expected Q-value for Equation (\ref{eq:final_policy}). Last line of Algorithm \ref{alg:Agent_get_action} then returns the action which maximizes the expected Q-value for the proposed policy. We further note that as the estimates of transition probabilities and Q-values are improved by the Algorithm \ref{alg:EQL}, the gap between the value of the proposed policy and the expected optimal value of the unobserved state can be bounded by Theorem \ref{final_theorem}.
\subsection{Complexity} \label{complexity}
At each time step $t$, Algorithm \ref{alg:Agent_get_action} computes the expected value of the Q-function whenever an action needs to be taken. This requires $\mathcal{O}\left(d|\mathcal{S}|^2+\log\left(|\mathcal{A}|\right)\right)$ computations. For our algorithm, we calculate the probability vector ${\bf \bar{p}}$ which is the conditional probability distribution of the states given the last known state $s_{t-d}$, and the sequence of actions $a_{t-d}, \cdots, a_{t-1}$. For time $t-d$, the true state is known and conditional probability becomes 
\begin{small}
\begin{align}
    {\bf \bar{p}}_{t-d}(s) &= 
        \Bigg\{\begin{array}{@{}cl}
                1,      & s = s_{t-d},\\
                0,   & otherwise
        \end{array}
\end{align}
\end{small}
Then, for each next time step, the probability vector is updated using the following recursion equation.
\begin{small}
\begin{align}
    {\bf \bar{p}}_{t-k} &= \left(\hat{p}(:, a_{t-(d-k)},:)\right)^T{\bf \bar{p}}_{t-k-1}\ \forall\ s\in\mathcal{S}, \forall\ 1\leq k<d
\end{align}
\end{small}
where $\hat{p}(:, a_{t-k-1}, :)$ is the state transition matrix of MDP $\mathcal{M}$ induced by action $a_{t-k-1}$. $\hat{p}(:, a_{t-k-1}, :)$ is obtained using Equation (\ref{eq:udpate_prob}). Since we do this update $d$ times, and each matrix multiplication costs $\mathcal{O}\left(|\mathcal{S}|^2\right)$, the total complexity to compute the state probability at each time step $t$ becomes $\mathcal{O}\left(d|\mathcal{S}|^2\right)$. Fetching the maximum element cost extra $\mathcal{O}\left(\log\left(|\mathcal{A}|\right)\right)$. The overall complexity at any time step $t$ thus becomes $\mathcal{O}\left(d|\mathcal{S}|^2+\log\left(|\mathcal{A}|\right)\right)$.
\begin{small}
	\begin{algorithm} [!htb]
		\small
		\begin{algorithmic}
			\STATE {\bf Input:} $\mathcal{S}$, $\mathcal{A}$, $\gamma$, $\epsilon_t$
			\FOR{$s \in \mathcal{S}$}
			    \FOR{$a \in \mathcal{A}$}
			        \STATE $\hat{p}(s,a,s') = 0, \hat{r}(s,a) = 0 $
			        \STATE $P(s,a,s') = 0, N(s,a) = 0 $
			        \STATE $R(s,a) = 0, Q(s,a) = 0$
			    \ENDFOR
			\ENDFOR
            \WHILE{1}
                \FOR{$s \in \mathcal{S}$}
                    \FOR{$a \in \mathcal{A}$}
                        \STATE Update probability and reward estimates
                            \begin{small}
                            \begin{align}
                        \hat{p}(s,a,s') &= \frac{P(s,a,s')}{\max{(1,N(s,a))}}\label{eq:udpate_prob}\\
                        \hat{r}(s,a) &= \frac{R(s,a)}{\max{(1,N(s,a))}}\label{eq:udpate_reward}
                            \end{align}
                            \end{small}
                    \ENDFOR
                    \STATE $V(s) = \max_{a}Q(s,a)$
                \ENDFOR
                \FOR{$s \in \mathcal{S}$}
                    \FOR{$a \in \mathcal{A}$}
                        \STATE Update Q function
                        \begin{small}
                        \begin{align}
                            Q(s,a) = \hat{r}(s,a) + \sum_{s'\in\mathcal{S}}\gamma V(s')\hat{p}(s,a,s') \label{eq:update_Q}
                        \end{align}
                        \end{small}
                    \ENDFOR
                \ENDFOR
                \STATE $t=0$
                \FOR{ $t = 1, 2, \cdots$}
                    \STATE $t+=1$
                    \STATE Known STATE $s_{t-d}$
                    \IF{$\mathbb{X}\sim U(0,1) < \epsilon_t$}
                        \STATE play $a_t$ randomly uniFORmly from $\mathcal{A}$
                    \ELSE
                        \STATE play $a_t$ = {\it Get\_EMQL\_Action}$\left(s_{t-d}, a_{t-d}, \cdots, a_{t-1}\right)$
                    \ENDIF
                    \IF{Observation $(s_{t-(d-1)}, r_{t-(d-1)})$ available}
                        \STATE $R(s_{t-d}, a_{t-d})\ += r_{t-(d-1)}$
                        \STATE $P(s_{t-d}, a_{t-d}, s_{t-(d-1)})\ += 1$
                        \STATE $N(s_{t-d}, a_{t-d})\ += 1$
                    \ENDIF
                \ENDFOR
            \ENDWHILE
		\end{algorithmic}
		\caption{{\it Expectation Maximization Q-Learning} (EMQL)}\label{alg:EQL}
	\end{algorithm}
	\begin{algorithm} [!htb]
		\small
		\begin{algorithmic}
			\STATE {\bf Input:} $\mathcal{S}$, $\mathcal{A}$, $\hat{p}$, $Q$, $\left(s_{t-d}, a_{t-d}, \cdots, a_{t-1}\right)$
			\STATE {\bf Output:} Estimated greedy action $a_t$
			\STATE ${\bf \bar{p}} = [0,\cdots,0]$, vector of length $|\mathcal{S}|$
			\STATE ${\bf \bar{p}}[s_{t-d}] = 1$
			\FOR{$0\leq k < d$}
			    \STATE ${\bf \bar{p}} = \left(\hat{p}(:, a_{t-(d-k)},:)\right)^T{\bf \bar{p}}$
			\ENDFOR
			\STATE ${\bf \bar{Q}} = {\bf \bar{p}}^T Q$
			\STATE Return $\arg\max_a {\bf \bar{Q}}$
		\end{algorithmic}
		\caption{{\it Get\_EMQL\_Action}}\label{alg:Agent_get_action}
	\end{algorithm}
\end{small}


\section{Evaluation} \label{evaluation}
We evaluate our algorithm EMQL on OpenAI Gym platform \cite{Brockman2016OpenAIG}. We consider the standard Frozen Lake ($8\times 8$ grid), and Cart Pole problem of OpenAI Gym which is as per the description in \cite{sutton2018reinforcement}. Frozen Lake environment has a discrete state space and Cart Pole environment has continuous state space with discrete action space.

We compare our algorithm with Extended MDP formulation by \cite{altman1992closed},  Model Based Simulation (MBS) algorithm of \cite{walsh2009learning}, and dQ algorithm of \cite{schuitema2010control} for constant delays. We also compare the proposed EMQL algorithm with MBS algorithm for stochastic delays. The metric of comparison is total cumulative reward accumulated at the end of each episode averaged over last 50 episodes. Exploration factor is time dependent and is chosen as $\epsilon_t = \frac{\log{\left(|\mathcal{S}|\sum_{s,a}N(s,a)+1\right)}}{\sum_{a}N(s,a)+1}\ \ \forall s,a$. This choice of exploration factor is same across all simulations. For stochastic delays, creating an augmented MDP is not feasible as delays can be arbitrarily large. We considered $50$ iterations of Frozen Lake environment, and $20$ iterations of Cart Pole environment. Each iteration is trained over $1000$ episodes.

For constant delays, we chose delays in the range of $d\in\{2,4\}$. For stochastic delays, each observation was independently delayed by delays generated using a geometric distribution with parameter $p$. The expected delay for this distribution is $\frac{1}{1-p}$. We note that, this may create asynchronous observations as delay $d_1$ of observation at $t_1$ may be higher than delay $d_2$ of observation at $t_2$, where $t_1 + d_1 > t_2+d_2$. The issue of asynchronous delays can be dealt by introducing time stamps in observations.

\subsection{Results} \label{eval_results}
Simulations results based for both Frozen Lake and Cart Pole environment are presented in Figure \ref{fig:Constant Delays} for constant delays. For stochastic delays the results are presented in Figure \ref{fig:Stochastic_delays}. In both figures, median of rewards in each iteration is plotted along with the top and bottom quantiles. 

\begin{figure}
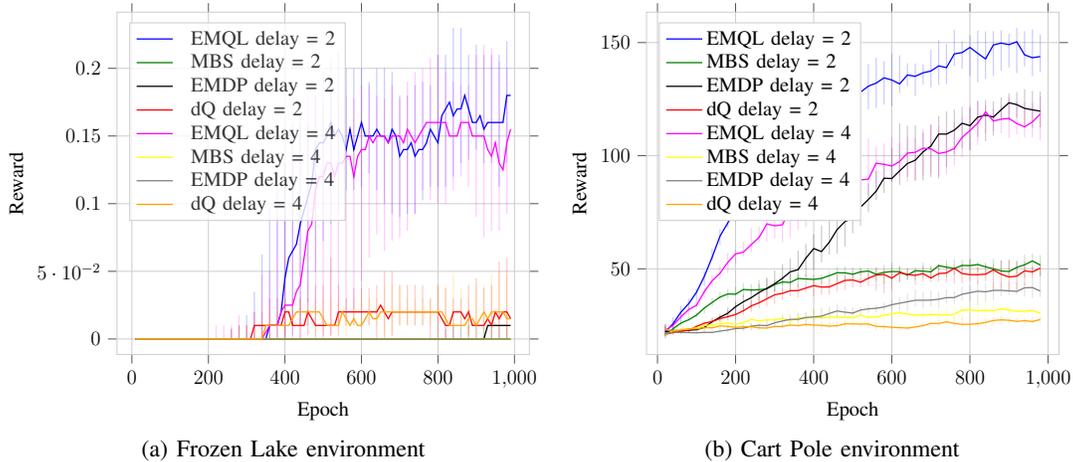

    \centering
    \subfloat[Frozen Lake environment]{%
        \input{figures/frozen_lake_constant.tex}%
        \label{fig:fl_fig_constant}
    } 
    \subfloat[Cart Pole environment]{%
        \input{figures/cart_pole_constant.tex}%
        \label{fig:cp_fig_constant}
    } 
    \caption{Reward accumulated by EMQL, MBS, EMDP, and dQ-learning algorithm with constant delays for Frozen Lake and Cart Pole environments. EMQL achieves higher average reward per episodes compared to other algorithms.}
    \label{fig:Constant Delays}
\end{figure}

\vspace{-0.05in}
As observed in Figure \ref{fig:Constant Delays}, total reward per episode is higher for EMQL algorithm compared to all other algorithms (MBS, dQ, EMDP) in the presence of constant delays. As delay increases, gap between the rewards also increase between the two algorithms. For MBS algorithm, this can be reasoned as - with large values of delays the state with largest likelihood might have lower probability of occurring.

For dQ algorithm, the reason for increasing difference in accumulated rewards can be credited to the memoryless algorithm. Also, Extended MDP algorithm is slow in convergence because of larger state space which grows exponentially. This reduction in convergence speed is visible in Figure \ref{fig:cp_fig_constant} where delay of 4 time steps cause a significant drop in performance of Extended MDP algorithm. 

For stochastic delays, the difference between the rewards accumulated by EMQL, and MBS algorithm increases when the probability of delay increase. For Frozen Lake environment, the proposed EMQL algorithm doesn't suffer much degradation even by increasing the expected delay from $1$ unit to $3$ units. However, MBS algorithm is not able to achieve similar performance for small expected delays.
In Cart Pole environment, EMQL algorithm beats the MBS algorithm significantly even when the stochastic delays are geometric distributed with expected delay of $1$ unit in Figure \ref{fig:cp_fig_sto}.

\begin{figure}
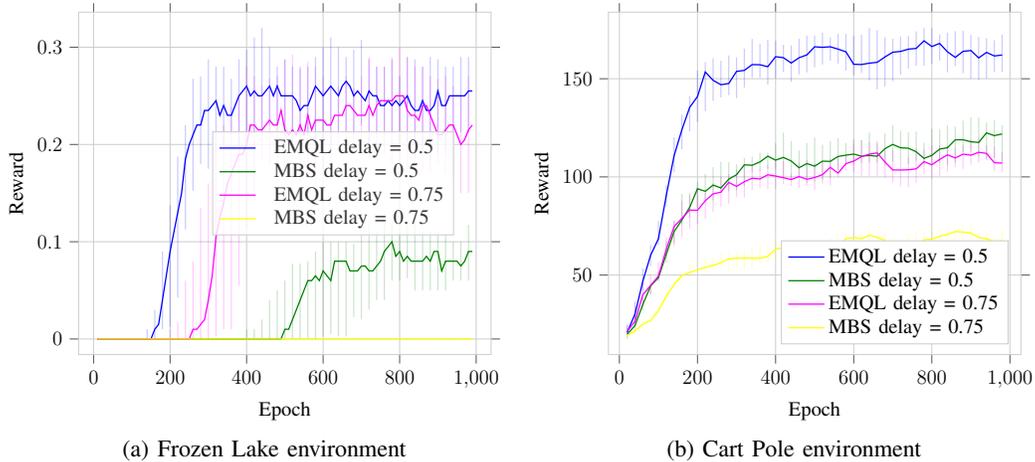

    \centering
    
    \subfloat[Frozen Lake environment]{%
        \input{figures/frozen_lake_stochastic.tex}%
        \label{fig:fl_fig_sto}
    } 
    \subfloat[Cart Pole environment]{%
        \input{figures/cart_pole_stochastic.tex}%
        \label{fig:cp_fig_sto}
    } 
    \caption{Reward accumulated by EMQL, MBS algorithms with stochastic delays for Frozen Lake and Cart Pole environments. EMQL achieves higher average reward per episodes compared to MBS algorithm when expected delays are high.}
    \label{fig:Stochastic_delays}
\end{figure}


\section{Conclusion}\label{sec:concl}
We considered the problem of delays in observation updates for a reinforcement learning agent. The current state of the environment is not immediately available to the agent. We proved that the expected immediate rewards generated for MDP with delays is same as expected immediate rewards generated for corresponding extended MDP without delays.  We proposed a new policy which can handle stochastic delays by optimizing on optimal Q-function of the true MDP. We then provided a lower bound on the value function for all states following the proposed policy. Based on this policy, we proposed a new algorithm, {\it Expectation Maximization Q-Learning} (EMQL), which is robust under constant, and stochastic delays. Using the knowledge of latest available state, sequence of actions, estimated transition probabilities, and reward distributions of the underlying MDP, we determine the best action which maximizes the expected reward for the unobserved state. Evaluations demonstrate the improvement over existing algorithms under constant, and stochastic delays. 


\bibliographystyle{plain}
\bibliography{ijcai19}

\end{document}